\documentclass{article}

\usepackage{iclr2026_conference,times}
\usepackage{amsmath,amssymb,amsthm,mathtools}
\usepackage{booktabs}
\usepackage{tabularx}
\usepackage{microtype}
\usepackage{hyperref}
\usepackage{url}
\hypersetup{
  pdftitle={Second-Order Muon Done Right: A Principled Marriage of Spectral Geometry and Curvature},
  pdfauthor={Tong Che}
}

\newcommand{\ShakespeareMuonLR}{0.018}
\newcommand{\ShakespeareGoLR}{0.080}
\newcommand{\ShakespeareMuonEarly}{1.70001}
\newcommand{\ShakespeareGoEarly}{1.63052}
\newcommand{\ShakespeareMuonMiddle}{1.50427}
\newcommand{\ShakespeareGoMiddle}{1.43613}
\newcommand{\ShakespeareMuonLate}{1.34825}
\newcommand{\ShakespeareGoLate}{1.31465}
\newcommand{\ShakespeareMuonTrajectory}{1.51072}
\newcommand{\ShakespeareGoTrajectory}{1.45471}
\newcommand{\ShakespeareEarlyGain}{4.09\%}
\newcommand{\ShakespeareMiddleGain}{4.53\%}
\newcommand{\ShakespeareLateGain}{2.49\%}
\newcommand{\ShakespeareTrajectoryGain}{3.71\%}

\newcommand{\ShakespeareSelectedGoLR}{0.080}

\newcommand{\ShakespeareSelectedGoEarly}{1.62170}

\newcommand{\ShakespeareSelectedGoMiddle}{1.43461}

\newcommand{\ShakespeareSelectedGoLate}{1.31396}
\newcommand{\ShakespeareSelectedMuonTrajectory}{1.51073}
\newcommand{\ShakespeareSelectedGoTrajectory}{1.45132}
\newcommand{\ShakespeareSelectedTrajectoryGain}{3.93\%}

\newcommand{\PTBMuonLR}{0.020}
\newcommand{\PTBGoLR}{0.080}
\newcommand{\PTBMuonTunedEarly}{1.46773}
\newcommand{\PTBGoTunedEarly}{1.45689}
\newcommand{\PTBMuonTunedMiddle}{1.28084}
\newcommand{\PTBGoTunedMiddle}{1.28167}
\newcommand{\PTBMuonTunedLate}{1.17560}
\newcommand{\PTBGoTunedLate}{1.17000}
\newcommand{\PTBMuonTrajectory}{1.30257}
\newcommand{\PTBGoTrajectory}{1.29757}
\newcommand{\PTBEarlyGain}{0.74\%}
\newcommand{\PTBLateGain}{0.48\%}
\newcommand{\PTBTrajectoryGain}{0.38\%}

\newcommand{\PTBHalfspaceEarly}{1.48882}
\newcommand{\PTBQuarterEarly}{1.45513}
\newcommand{\PTBMapbackGain}{2.26\%}

\newcommand{\PTBRefreshEarly}{1.45513}
\newcommand{\PTBLazyEarly}{1.45549}
\newcommand{\PTBRefreshMiddle}{1.28378}
\newcommand{\PTBLazyMiddle}{1.28259}
\newcommand{\PTBRefreshLate}{1.17238}
\newcommand{\PTBLazyLate}{1.16917}
\newcommand{\PTBLazyMiddleGain}{0.00119}
\newcommand{\PTBLazyLateGain}{0.00320}
\newcommand{\PTBLazyStepRatio}{0.798$\times$}
\newcommand{\PTBLazyStepSaving}{20.2\%}

\newcommand{\ModMuonLR}{0.160}
\newcommand{\ModGoLR}{0.080}
\newcommand{\ModOneOhThreeMuonTrain}{66}
\newcommand{\ModOneOhThreeGoTrain}{64}
\newcommand{\ModOneOhThreeMuonGrok}{2320}
\newcommand{\ModOneOhThreeGoGrok}{290}
\newcommand{\ModOneOhThreeMuonDelay}{2254}
\newcommand{\ModOneOhThreeGoDelay}{226}
\newcommand{\ModOneOhThreeSpeedup}{8.0}
\newcommand{\ModOneOhSevenMuonTrain}{65}
\newcommand{\ModOneOhSevenGoTrain}{68}
\newcommand{\ModOneOhSevenMuonGrok}{4520}
\newcommand{\ModOneOhSevenGoGrok}{220}
\newcommand{\ModOneOhSevenMuonDelay}{4455}
\newcommand{\ModOneOhSevenGoDelay}{152}
\newcommand{\ModOneOhSevenSpeedup}{20.5}
\newcommand{\ModOneOhSevenLastRetentionFailure}{4470}
\newcommand{\ModOneOhSevenRetentionReturn}{4480}

\newcommand{\method}{\textsc{Go-MUON}}

\newcommand{\polar}{\operatorname{Polar}}
\newcommand{\op}{\mathrm{op}}
\newcommand{\tr}{\operatorname{tr}}
\newcommand{\E}{\mathbb{E}}
\newcommand{\R}{\mathbb{R}}
\newcommand{\ip}[2]{\left\langle #1,#2\right\rangle}
\newcommand{\kron}{\otimes}
\newcolumntype{Y}{>{\raggedright\arraybackslash}X}

\newtheorem{theorem}{Theorem}
\newtheorem{corollary}[theorem]{Corollary}
\newtheorem{proposition}[theorem]{Proposition}

\title{Second-Order Muon Done Right:\\
A Principled Marriage of Spectral Geometry and Curvature}

\author{Tong Che \\
NVIDIA Research \\
\texttt{tongc@nvidia.com}}

\iclrfinalcopy

\begin{document}

\maketitle
\lhead{\textsc{Second-Order Muon Done Right}}

\begin{abstract}
Muon's polar update is exact for an unweighted spectral geometry.  We
introduce \method, which uses a matched data-dependent geometry and reuses it
across several optimization steps.  Conditioned on any positive-definite
left and right maps, its raw update exactly solves the corresponding weighted
spectral oracle; this statement is independent of how the maps are estimated
or how recently they were refreshed.  For softmax cross-entropy, we quantify
when the observed-label backward factor approaches the model Fisher and
generalized Gauss--Newton factor.  We also show that four-step refresh nearly
preserves the tracking delay of slowly changing geometry while increasing
stationary factor noise, making lazy geometry a compute--statistics tradeoff
rather than a denoising mechanism.  In a hash-matched Penn Treebank
comparison, \method{} lowers the geometric mean of three fixed training-loss
windows by \PTBTrajectoryGain{} across three paired seeds.  On Tiny
Shakespeare, a completed three-seed comparison lowers the
same trajectory summary by \ShakespeareTrajectoryGain{} against the
registered Muon control.  Four-step refresh also improves the late PTB
training window and reduces mean step time by about \(20\%\).  On corrected
independent-CLS modular addition, a formal one-seed-per-modulus screen finds
that \method{} reaches sustained \(99\%\) held-out accuracy at update
\ModOneOhThreeGoGrok{} versus
\ModOneOhThreeMuonGrok{} for Muon on modulus 103, and at update
\ModOneOhSevenGoGrok{} versus \ModOneOhSevenMuonGrok{} on modulus 107.
\end{abstract}

\section{Introduction}

For a matrix parameter \(W\in\R^{m\times n}\), Muon applies a polar map to a
momentum source \(M\) \citep{jordan2024muon}.  Up to the sign convention of the
parameter update, the direction solves
\begin{equation}
  \max_{\|D\|_{\op}\leq 1}\ip{M}{D}
  =\|M\|_*,
  \qquad D=\polar(M).
  \label{eq:muon}
\end{equation}
This moves every nonzero singular direction equally under an operator-norm
update budget.

The same construction extends to any declared symmetric positive-definite
(SPD) coordinate maps \(P_A,P_B\).  The matched direction
\[
 D_0=P_B\polar(P_BMP_A)P_A
\]
is the exact support direction of the ball
\(\|P_B^{-1}DP_A^{-1}\|_{\op}\leq1\).  This separates two questions that are
often conflated: the oracle is exact for the supplied geometry, while the
choice and estimation of that geometry determine whether the direction is
useful.  Fisher semantics, Kronecker factorization, and factor freshness are
not assumptions of the oracle statement.

\method{} instantiates the maps with inverse fourth roots of OAS-regularized
activation and backpropagated-gradient second moments.  Under softmax
cross-entropy these observed-label factors approach model-Fisher/GGN factors
when minibatch error vanishes and the predictive distribution approaches the
data conditional.  The quarter power tempers factor anisotropy relative to a
full inverse, while a positive Frobenius graft restores Muon's update energy.

The geometry is refreshed every \(K=4\) updates.  Reuse steps skip activation
retention, output-gradient hooks, moment construction, shrinkage, and factor
eigendecompositions, while gradient and momentum information remains current.
The coefficient \(\beta_B^K\) preserves elapsed-step decay of the retained
backward state.  It does not reproduce the omitted minibatch averages: under
stationary noise the resulting estimator is noisier than a per-step EMA, even
though its average lag on a slowly varying signal is nearly unchanged.

Our contributions are:
\begin{enumerate}
  \item a factor-agnostic formulation of the matched spectral oracle, together
        with quantitative conditioning, grafted-alignment, and cached-geometry
        bounds for the implemented direction;
  \item a conditional empirical-Fisher-to-Fisher/GGN result for softmax
        backward factors and an exact variance--lag characterization of sparse
        \(\beta_B^K\) refresh;
  \item the \method{} system, which combines OAS K-FAC-style factors with
        hook-gated factor refresh, cached inverse fourth roots, and fresh
        gradient and momentum updates;
  \item paired character-language-model experiments showing a
        \PTBTrajectoryGain{} fixed-window trajectory reduction against Muon
        in a hash-matched PTB comparison, a
        \ShakespeareTrajectoryGain{} reduction against
        the registered Tiny Shakespeare Muon control, and a \(0.798\times\)
        mean-step-time ratio for four-step refresh; and
  \item a corrected independent-CLS modular-addition screen in which
        \method{} reaches sustained \(99\%\) held-out accuracy at updates
        \ModOneOhThreeGoGrok{} and \ModOneOhSevenGoGrok{}, versus
        \ModOneOhThreeMuonGrok{} and \ModOneOhSevenMuonGrok{} for Muon.
\end{enumerate}

\section{Matched spectral geometry}

\subsection{A factor-agnostic oracle}

Let \(P_B\in\mathbb S_{++}^m\), \(P_A\in\mathbb S_{++}^n\), and define
\begin{equation}
 C=P_BMP_A,\qquad
 q_P(D)=\|P_B^{-1}DP_A^{-1}\|_{\op}.
 \label{eq:weighted-norm}
\end{equation}
For \(\rho\geq0\), consider the support problem
\begin{equation}
 \max_D\ip{M}{D}
 \quad\text{subject to}\quad q_P(D)\leq\rho.
 \label{eq:lmo}
\end{equation}

\begin{theorem}[Matched spectral oracle]
\label{thm:lmo}
Let \(Q(C)\) be any support polar satisfying
\[
 Q(C)\in\arg\max_{\|U\|_{\op}\leq1}\ip{C}{U}.
\]
Then the value of Equation~\eqref{eq:lmo} is \(\rho\|C\|_*\), and one
maximizer is
\begin{equation}
 \boxed{D_\rho=\rho P_BQ(C)P_A.}
 \label{eq:direction}
\end{equation}
At unit radius, write \(D_0:=D_1\); then
\begin{equation}
 \boxed{\ip{M}{D_0}=\ip{C}{Q(C)}=\|C\|_*.}
 \label{eq:identity}
\end{equation}
For \(M\ne0\), this quantity is strictly positive.  Under the conventional
minimization definition, \(-D_\rho\) is the LMO; the optimizer equivalently
subtracts the positive-alignment direction \(D_\rho\).
\end{theorem}

\begin{proof}
Set \(U=P_B^{-1}DP_A^{-1}\), so \(D=P_BUP_A\).  Cyclicity of the trace gives
\[
 \ip{M}{D}=\ip{P_BMP_A}{U}=\ip{C}{U}.
\]
The constraint is \(\|U\|_{\op}\leq\rho\).  Operator/nuclear duality gives
the value \(\rho\|C\|_*\), and mapping a support polar back gives
Equation~\eqref{eq:direction}.  This is the matrix specialization of the
general preconditioned-norm pullback oracle
\citep{veprikov2025preconditioned}.
\end{proof}

\paragraph{Polar selection.}
If \(C\) has full rank \(k=\min(m,n)\), its support polar is unique.  At rank
deficiency, the support maximizer is not unique: the canonical partial polar
sets the null-space completion to zero, whereas the compact full-SVD polar
used in our implementation supplies a semi-orthogonal completion.  Both solve
Equation~\eqref{eq:lmo}, but the completion can affect a later Frobenius
graft.  For \(M=0\), every feasible point is optimal; the implementation
defines \(D_0=0\) and skips the graft.

\begin{corollary}[Cached geometry]
\label{cor:cached}
For any cached SPD maps \(\widetilde P_A,\widetilde P_B\) and any \(M\ne0\),
\begin{equation}
 \ip{M}{
 \widetilde P_B
 Q(\widetilde P_BM\widetilde P_A)
 \widetilde P_A}
 =\|\widetilde P_BM\widetilde P_A\|_*>0.
 \label{eq:stale}
\end{equation}
\end{corollary}

Corollary~\ref{cor:cached} certifies exactness for the cached ball, not
closeness to the current ball.  With momentum, \(M\) is the source passed to
the polar map; current-gradient alignment and finite-step loss change are
separate properties.  A scale-free comparison with current maps is possible.
For
\begin{equation}
 \Theta_t=
 \kappa_2(P_{B,t}^{-1}\widetilde P_B)\,
 \kappa_2(P_{A,t}^{-1}\widetilde P_A),
 \label{eq:stale-distortion}
\end{equation}
Proposition~\ref{prop:stale-approx} shows that the cached direction, rescaled
to its current weighted radius, achieves at least a
\(\Theta_t^{-1}\) fraction of the current oracle value.  This is conditional
on coordinate-map drift; the present experiments do not measure
\(\Theta_t\).

\subsection{Why the map-back matters}

Conditioning only before the polar map need not preserve source alignment.
For example, take
\[
 A=B=\operatorname{diag}(1/9,9),\qquad
 M=\begin{bmatrix}-1&-4\\-2&-6\end{bmatrix}.
\]
The inside-only score and its polar are
\[
 Z=B^{-1/2}MA^{-1/2}
  =\begin{bmatrix}-9&-4\\-2&-2/3\end{bmatrix},\qquad
 \polar(Z)=\frac1{\sqrt{949}}
 \begin{bmatrix}-25&-18\\-18&25\end{bmatrix},
\]
so
\begin{equation}
 \ip{M}{\polar(Z)}=-\frac{17}{\sqrt{949}}<0.
 \label{eq:inside-only-counterexample}
\end{equation}
A positive scalar graft cannot change this sign.  The matched map-back instead
optimizes in transformed coordinates and returns to the original tangent
coordinates, yielding the positive identity in Equation~\eqref{eq:identity}.

\section{K-FAC-root instantiation}

\subsection{Observed-label factors and their interpretation}

For a linear layer \(y=Wx\), the implementation forms minibatch second moments
\begin{equation}
 \widehat A=\frac1N\sum_{i=1}^N x_ix_i^\top,\qquad
 \widehat B=\frac1N\sum_{i=1}^N\delta_i\delta_i^\top,\qquad
 \delta_i=\frac{\partial\ell_i}{\partial y_i}.
 \label{eq:factors}
\end{equation}
These are activation and observed-label backpropagated-gradient factors.
K-FAC approximates an expected Kronecker block by a product of expected
factors \citep{martens2015kfac}.  Calling \(\widehat B\) a Fisher factor
requires an additional relation between the data conditional and the model
predictive distribution \citep{kunstner2019limitations}.

For softmax cross-entropy, that relation can be quantified exactly.  Fix an
input \(x\), write \(p=p_\theta(\cdot\mid x)\) and
\(q=q(\cdot\mid x)\), and let
\(\delta_z=p-e_y\) be the logit gradient.  Define
\[
 B_{q,z}=\E_{y\sim q}[\delta_z\delta_z^\top],\qquad
 B_{F,z}=\E_{y\sim p}[\delta_z\delta_z^\top]
 =\operatorname{Diag}(p)-pp^\top.
\]
For logits as natural parameters, \(B_{F,z}\) is also the output GGN factor
\citep{martens2020new}.

\begin{proposition}[Observed-label factor mismatch]
\label{prop:ef}
Let \(r=q-p\).  Then
\begin{equation}
 \boxed{
 B_{q,z}-B_{F,z}
 =\operatorname{Diag}(r)-pr^\top-rp^\top,
 \qquad
 \|B_{q,z}-B_{F,z}\|_2\leq3\|q-p\|_2.}
 \label{eq:ef-bound}
\end{equation}
If \(\delta_l=J_l^\top\delta_z\), where \(J_l\) is independent of \(y\), then
\begin{equation}
 \boxed{
 \|B_{q,l}-B_{F,l}\|_2
 \leq3\|J_l\|_2^2\|q-p\|_2.}
 \label{eq:hidden-ef-bound}
\end{equation}
\end{proposition}

\begin{proof}
Expanding
\(\E_{y\sim q}[(p-e_y)(p-e_y)^\top]\) and subtracting
\(\operatorname{Diag}(p)-pp^\top\) gives the identity.
Then
\(\|\operatorname{Diag}(r)\|_2\leq\|r\|_2\) and
\(\|pr^\top\|_2=\|p\|_2\|r\|_2\leq\|r\|_2\).
The hidden-layer result follows by congruence with \(J_l\).
\end{proof}

\begin{corollary}[Conditional factor consistency]
\label{cor:ef-consistency}
Consider a sequence \(\theta_j\) for which
\(p_{\theta_j}(\cdot\mid x)\to q(\cdot\mid x)\) and
\(\sup_j\|J_l(\theta_j)\|_2<\infty\).  If a minibatch estimator
\(\widehat B_{j,l}\) is consistent for \(B_{q,l}(\theta_j)\), then
\[
 \|\widehat B_{j,l}-B_{F,l}(\theta_j)\|_2
 \xrightarrow{p}0
\]
as both the effective sample size and \(j\) tend to infinity.
\end{corollary}

Corollary~\ref{cor:ef-consistency} requires realizability or predictive
consistency; it is not a property of arbitrary training trajectories.
It also applies only to the backward factor: the full block still incurs the
K-FAC approximation
\[
 \E[(xx^\top)\kron(\delta\delta^\top)]
 \;\not=\;
 \E[xx^\top]\kron\E[\delta\delta^\top]
\]
in general.
Moreover, the softmax Fisher always has the logit-shift null direction, and
its remaining eigenvalues can vanish as \(p\) approaches a point mass.
Absolute factor convergence therefore does not imply relative inverse-root
accuracy near interpolation; the effective floor and conditioning bounds
below must be included when analyzing the resulting geometry.

\subsection{Quarter-power geometry and Frobenius graft}

Let \(A\succ0\), \(B\succ0\) denote the effective regularized factors.  For
\(0\leq\alpha\leq1/2\), define
\begin{align}
 C_\alpha&=B^{-\alpha}MA^{-\alpha},\\
 D_{\alpha,0}&=B^{-\alpha}Q(C_\alpha)A^{-\alpha},\\
 \widehat D_\alpha
 &=D_{\alpha,0}\frac{\sqrt{k}}{\|D_{\alpha,0}\|_F},
 \qquad k=\min(m,n).
 \label{eq:graft}
\end{align}
The cases \(\alpha=0,1/4,1/2\) give Muon, the \method{} quarter-power
geometry, and full-factor spectral map-back, respectively.  The last is a
spectral oracle in full-factor coordinates, not a natural-gradient step.

For \(\alpha=1/4\),
\[
 \|B^{1/4}DA^{1/4}\|_{\op}\leq1
\]
is the operator-norm analogue of the square-root K-FAC quadratic geometry,
since
\begin{equation}
 \operatorname{vec}(D)^\top
 (A^{1/2}\kron B^{1/2})\operatorname{vec}(D)
 =\|B^{1/4}DA^{1/4}\|_F^2.
 \label{eq:root-fisher}
\end{equation}
Replacing the Frobenius norm by the operator norm is a design choice.  Under
the conditions of Proposition~\ref{prop:ef}, this admits a root-Fisher/GGN
interpretation; the oracle itself requires only SPD factors.

\begin{proposition}[Grafted conditioning and source alignment]
\label{prop:conditioning}
Let \(Q\) be the compact full-SVD support polar used by the implementation and
let \(K_{AB}=\kappa_2(A)\kappa_2(B)\).  Let \(\kappa_+\) denote the ratio of
largest to smallest nonzero singular value.  For \(M\ne0\),
\begin{equation}
 \boxed{
 \ip{M}{\widehat D_\alpha}
 \geq K_{AB}^{-\alpha}\|M\|_*,
 \qquad
 \kappa_+(\widehat D_\alpha)\leq K_{AB}^{\alpha}.}
 \label{eq:conditioning-bound}
\end{equation}
\end{proposition}

The first bound compares equal-Frobenius-radius source alignment with Muon's
value \(\|M\|_*\).  At \(\alpha=1/4\), the retained fraction is at least
\([\kappa(A)\kappa(B)]^{-1/4}\).  This is a worst-case conditioning
certificate, not a pointwise ordering of Muon, quarter-power, and half-power
directions.  Appendix~\ref{app:proofs} gives the proof.

The positive graft preserves the oracle ray and sign:
\[
 \ip{M}{\widehat D_\alpha}
 =\frac{\sqrt{k}}{\|D_{\alpha,0}\|_F}\|C_\alpha\|_*>0.
\]
It does not preserve the original unit-radius LMO.  Its weighted spectral
radius is the source-dependent value
\(\sqrt{k}/\|D_{\alpha,0}\|_F\).

The grafted direction is invariant to independent positive scalar rescalings
of \(A\) and \(B\).  With the canonical polar---and with the compact polar at
full rank---it is also equivariant under orthogonal changes of the input and
output bases: replacing
\((M,B,A)\) by
\((UMV^\top,UBU^\top,VAV^\top)\) sends
\(\widehat D_\alpha\) to \(U\widehat D_\alpha V^\top\).
At rank deficiency, a full-SVD null-space completion can break this
equivariance.  General affine reparameterization invariance does not follow
because matrix powers and the polar map are not equivariant under arbitrary
congruences \citep{luk2018coordinate}.

\subsection{OAS factors and numerical SPD maps}

Finite-minibatch factors can be singular.  For a \(d\times d\) empirical
second moment \(S\) built from \(N\) samples, we use Oracle Approximating
Shrinkage (OAS) \citep{chen2010oas}:
\begin{align}
 \lambda(S,N)
 &=\operatorname{clip}_{[0,1]}
 \frac{(1-2/d)\tr(S^2)+\tr(S)^2}
 {(N+1-2/d)(\tr(S^2)-\tr(S)^2/d)}, \\
 \operatorname{OAS}(S)
 &=(1-\lambda)S+\lambda\frac{\tr(S)}{d}I.
 \label{eq:oas}
\end{align}
When \(\lambda>0\) and \(\tr(S)>0\),
\begin{equation}
 \kappa_2(\operatorname{OAS}(S))
 \leq\frac{d-(d-1)\lambda}{\lambda}.
 \label{eq:oas-condition}
\end{equation}
The implementation additionally floors eigenvalues at roundoff scale before
taking inverse fourth roots.  The resulting numerical SPD matrices, rather
than the potentially singular raw moments, are the factors to which
Theorem~\ref{thm:lmo} and Proposition~\ref{prop:conditioning} apply.

Writing \(B_t^{\rm eff}\) and \(\bar B_t^{\rm eff}\) for the current and
cached factors after shrinkage and the root routine's eigenvalue floor,
Fisher-factor error has four distinct sources:
\begin{equation}
\begin{split}
 \|\bar B_t^{\rm eff}-B_{F,t}\|_2
 \leq{}&
 \underbrace{\|\widehat B_t-B_{q,t}\|_2}_{\text{finite sample}}
 +\underbrace{\|B_{q,t}-B_{F,t}\|_2}_{\text{model mismatch}}\\
 &+\underbrace{\|B_t^{\rm eff}-\widehat B_t\|_2}_{\text{shrinkage/floor}}
 +\underbrace{\|\bar B_t^{\rm eff}-B_t^{\rm eff}\|_2}
 _{\text{EMA and staleness}}.
 \label{eq:factor-error}
\end{split}
\end{equation}
The decomposition separates the conditional Fisher interpretation from the
regularization and temporal choices used by the optimizer.

\subsection{Sparse refresh is a variance--lag tradeoff}

At refresh step \(t\), we compute
\begin{equation}
 A_t=\operatorname{OAS}(\widehat A_t),\qquad
 B_t=\operatorname{OAS}(\widehat B_t).
\end{equation}
The first \(B_t\) initializes \(\bar B\).  Later refreshes use
\begin{equation}
 \boxed{
 \bar B_t=\beta_B^K\bar B_{t-K}+(1-\beta_B^K)B_t.
 }
 \label{eq:lazy-ema}
\end{equation}
The exponent applies \(K\) elapsed steps of decay to the retained state.  It
does not reproduce \(K\) per-step EMA observations.  To characterize the
difference, let \(X_t\) be a centered post-OAS factor observation and compare
\[
 Z_t=\beta_BZ_{t-1}+(1-\beta_B)X_t
\]
with the refresh-boundary process
\[
 Y_j=\beta_B^K Y_{j-1}+(1-\beta_B^K)X_{jK}.
\]

\begin{proposition}[Sparse-EMA covariance and lag]
\label{prop:lazy-ema}
If \(\operatorname{vec}(X_t)\) is weakly stationary with finite covariance
and both recursions are in steady state, then at refresh boundaries
\begin{equation}
 \boxed{\operatorname{Cov}(\operatorname{vec}Y_j)
 \succeq
 \operatorname{Cov}(\operatorname{vec}Z_{jK}).}
 \label{eq:covariance-dominance}
\end{equation}
For temporally IID observations with covariance \(\Sigma\),
\begin{align}
 \operatorname{Cov}(\operatorname{vec}Y)
 &=\frac{1-\beta_B^K}{1+\beta_B^K}\Sigma,\\
 \operatorname{Cov}(\operatorname{vec}Z)
 &=\frac{1-\beta_B}{1+\beta_B}\Sigma.
 \label{eq:iid-variance}
\end{align}
For an affine signal observed without noise, the phase-averaged lag of the
held estimator is
\begin{equation}
 L_K=\frac{K\beta_B^K}{1-\beta_B^K}+\frac{K-1}{2}.
 \label{eq:lazy-lag}
\end{equation}
\end{proposition}

At \(\beta_B=0.97,K=4\), the sparse-to-per-step variance ratio is
\(3.99537\), while \(L_4=32.3714\) steps versus \(L_1=32.3333\).
Thus sparse refresh nearly preserves average slow-signal lag while retaining
about one quarter as many independent factor observations.  Sampling occurs
before the refresh-domain EMA, so high-frequency variation can alias.  The
activation factor \(A_t\) has no EMA and is simply sampled and held.

Holding one geometry within each four-step block creates blockwise geometry
coherence: the preconditioner does not rotate between minibatches inside the
block.  Whether this coherence stabilizes the optimizer trajectory is an
empirical mechanism hypothesis, not a consequence of
Proposition~\ref{prop:lazy-ema}.

\subsection{Complete update}

With \(K=4\), \(\beta_B=0.97\), momentum coefficient \(\beta\), and a
Nesterov-style source, one optimizer step is:
\begin{align}
 \text{if }t\bmod K=0:\quad
 &A_t,B_t\leftarrow\text{OAS K-FAC statistics}, \\
 &\bar B_t\leftarrow
   \beta_B^K\bar B_{t-K}+(1-\beta_B^K)B_t, \\
 &P_A\leftarrow A_t^{-1/4},\qquad
   P_B\leftarrow\bar B_t^{-1/4}; \\
 M_t&=\beta M_{t-1}+(1-\beta)G_t, \\
 \widehat M_t&=(1-\beta)G_t+\beta M_t, \\
 C_t&=P_B\widehat M_tP_A, \\
 D_{0,t}&=P_B\polar(C_t)P_A, \\
 D_t&=D_{0,t}\frac{\sqrt{\min(m,n)}}{\|D_{0,t}\|_F}, \\
 W_{t+1}&=W_t-\eta_t\sqrt{\max(1,m/n)}D_t.
 \label{eq:optimizer}
\end{align}
The learning rate \(\eta_t\) follows a cosine schedule.  Embeddings,
normalization parameters, and other non-matrix parameters use AdamW
\citep{kingma2015adam}.

The implementation gates its forward hooks before the model forward pass.
On refresh steps it retains linear-layer inputs and registers output-gradient
hooks.  On reuse steps the hook returns immediately, bypassing activation
storage, delta hooks, moment formation, OAS, and factor eigendecomposition.
Momentum, the central polar map, grafting, and the parameter update execute
every step.

\subsection{Cost}

For a layer of shape \(m\times n\), a refresh computes moments and matrix
functions of the \(m\times m\) and \(n\times n\) factors.  Reuse eliminates
those operations and their telemetry while retaining the matrix products and
polar map needed for the direction.  A refresh fraction of \(1/K\) amortizes
factor collection and decomposition across \(K\) updates.  The experiments
measure the resulting savings end to end.

\section{Related work}

\subsection{Muon and norm-based matrix optimization}

Muon applies an approximate polar map to momentum matrices, using Newton--
Schulz iterations to make matrix updates approximately semi-orthogonal
\citep{jordan2024muon}.  The norm-based account of matrix optimization
identifies this map as steepest descent under an operator-norm constraint,
whose dual objective is the nuclear norm
\citep{bernstein2024oldnorm}.  Modular norm optimization extends this viewpoint
from individual matrices to composed neural-network modules
\citep{large2024modular}; Scion develops norm-constrained linear minimization
oracles more generally \citep{pethick2025scion}; and Gluon connects practical
Muon-like updates with convergence theory for LMO-based optimization
\citep{riabinin2025gluon}.  Subsequent systems work implements Muon at
language-model scale \citep{liu2025scalable}.

Several variants refine Muon's normalization while retaining a polar or
orthogonalization core.  Muon+ adds an additional normalization step
\citep{zhang2026muonplus}, while MuonEq uses inexpensive equilibration before
orthogonalization \citep{chang2026muoneq}.  These methods modify the matrix
presented to the polar map.  \method{} instead instantiates the general
preconditioned-norm pullback oracle: the same SPD maps transform the source
into polar coordinates and map the support direction back.  Its positive
Frobenius graft restores Muon's update energy while retaining that oracle ray.

\subsection{Kronecker curvature and structured preconditioning}

K-FAC approximates a layerwise model-Fisher block by the Kronecker product of
an activation factor and a backward factor \citep{martens2015kfac}.  It has been
scaled through distributed factor formation, inversion, and communication
\citep{osawa2019distributed,pauloski2020distributed}; SINGD gives an
inverse-free, numerically stable structured natural-gradient formulation
\citep{lin2024singd}.  Full Gauss--Newton experiments on small language models
demonstrate accurate layerwise curvature as a strong optimization oracle
\citep{abreu2025fullgn}.  \method{} uses observed-label second moments rather
than model-sampled Fisher factors.  Proposition~\ref{prop:ef} quantifies when
the backward factor approaches the Fisher/GGN factor; the distinction remains
essential away from predictive consistency \citep{kunstner2019limitations}.
K-FAC's natural-gradient update retains invariance to affine activation
reparameterizations under its metric construction
\citep{luk2018coordinate}.  The quarter-power spectral oracle retains only
the orthogonal and scalar equivariances stated in Section~3.

Shampoo uses Kronecker-factored matrix moments and inverse matrix powers for
large-scale preconditioning \citep{gupta2018shampoo}; scalable Shampoo
amortizes these matrix operations in distributed training
\citep{anil2020scalable}.  SOAP runs an Adam-like method in Shampoo's evolving
eigenbasis \citep{vyas2024soap}.  The general theory of preconditioned matrix
norms makes explicit how transformations inside and outside an LMO determine
the corresponding steepest-descent geometry
\citep{veprikov2025preconditioned}.  Theorem~\ref{thm:lmo} is a Kronecker
matrix instance of this general pullback, not a new duality principle.  In
\method{}, OAS-regularized activation and backward second moments select the
geometry; Shampoo uses gradient moments.  OAS was introduced as an automatic
shrinkage estimator for finite-sample covariance matrices
\citep{chen2010oas}; here it supplies isotropic shrinkage before the numerical
eigenvalue floor.

\subsection{Curvature-aware Muon variants}

Mousse applies a quarter-power matched polar map with a Frobenius graft to
Shampoo-style gradient-Gram geometry \citep{zhang2026mousse}.  It updates
factors every step, refreshes matrix functions periodically, and uses an
inverse-power exponent below one quarter in its reported language-model
experiments.  The quarter-power matched-map structure and graft are shared.
\method{} differs in its OAS activation/backward factor field, exact
quarter-power instance, and hook-gated schedule that skips factor collection
on reuse steps.

FISMO derives matched map-back for a Fisher-structured spectral trust region,
yielding inverse half powers inside and outside the polar map
\citep{xu2026fismo}.  \method{} uses the \emph{square root} of the K-FAC block,
yielding inverse fourth powers when the observed-label factors admit the
conditional Fisher interpretation.  Both are instances of matched spectral
map-back at different exponents.  Corollary~\ref{cor:cached} gives positive
source alignment for any cached SPD maps, independent of their Fisher
semantics.

Other recent methods insert cheaper curvature information at different points
in the Muon pipeline.  Newton--Muon corrects the input side using activation
moments before orthogonalization \citep{du2026newtonmuon}; Muon2 uses Adam-like
second moments before orthogonalization and couples this with efficient polar
computation \citep{liu2026muon2}; and MALT uses lightweight diagonal
preconditioning \citep{wu2026malt}.  Empirical studies of whitening optimizers
similarly ask which parts of whitening, momentum, and normalization account
for their gains \citep{frans2025whitening}.  \method{} makes this orientation
change coordinate-consistent: the preconditioned matrix enters the polar map
and returns through the matching tangent map.  The exact oracle is general;
the OAS K-FAC-root field is the \method{} instantiation.

Table~\ref{tab:related-work} compares update construction and curvature reuse
across representative structured matrix optimizers.

\begin{table}[t]
 \caption{Relationship to representative structured matrix optimizers.
 ``Factor skip'' denotes steps that skip accumulation of the underlying
 second-moment factors.}
 \label{tab:related-work}
 \centering
 \footnotesize
 \setlength{\tabcolsep}{3pt}
 \begin{tabularx}{\textwidth}{lYYY}
  \toprule
  Method & Statistical geometry & Matrix update & Curvature amortization \\
  \midrule
  Muon & None & \(\polar(M)\) & --- \\
  K-FAC & Activation/backward Fisher factors & Approximate natural gradient & Periodic inverse; later systems also skip factors \\
  Scalable Shampoo & Gradient Gram moments & Kronecker inverse-power preconditioning & Periodic roots; optional factor skip \\
  SOAP & Shampoo moments and eigenbasis & Adam in the Shampoo basis & Periodic eigenbasis refresh \\
  Mousse & Shampoo gradient moments & Root-metric matched polar with graft & Factors frequent, matrix functions periodic \\
  FISMO & Fisher/K-FAC factors & Fisher-metric matched spectral step & Factors updated in its main algorithm \\
  \method{} & OAS observed-label activation/backward moments & Quarter-power matched polar with graft & Hook-gated factors and roots every \(K\); fresh source every step \\
  \bottomrule
 \end{tabularx}
\end{table}

\subsection{Amortized and stale curvature}

Prior second-order methods separate the frequency of gradient updates from the
frequency of curvature work.  The original online K-FAC algorithm
updates moving-average factors from minibatches while reusing a factorization
for multiple iterations \citep{martens2015kfac}.  Distributed K-FAC subsequently
moved inverse computation and factor-statistic computation to asynchronous
workers, explicitly accepting stale curvature in exchange for throughput
\citep{ba2017distributedkfac}.  Large-scale implementations made both
timescales configurable: they reduce factor formation as well as
eigendecomposition frequency \citep{osawa2019distributed,pauloski2020distributed}.
KAISA applies this separation to Transformer linear layers, including BERT,
with distinct factor-update and eigendecomposition-update intervals
\citep{pauloski2021kaisa}.  These systems schedule factor updates and
eigendecompositions independently, including gated activation and backward
capture at Transformer scale.

The Shampoo lineage makes a related separation.  Original Shampoo accumulates
matrix moments more often than it computes inverse roots
\citep{gupta2018shampoo}.  Scalable Shampoo performs roots asynchronously and
also exposes a separate interval for skipping gradient-statistic accumulation
\citep{anil2020scalable}.  Distributed Shampoo typically updates factors every
step while refreshing inverse roots periodically
\citep{shi2023distributedshampoo}.  SOAP treats eigenbasis frequency as an
explicit systems parameter, while later analysis separates stale eigenvalues
from stale eigenvectors and develops adaptive refresh criteria
\citep{vyas2024soap,eschenhagen2025purifying}.  Together, these methods
separate factor, matrix-function, and gradient-update timescales.

\method{} couples quarter-power matched geometry to a two-timescale
sample-and-hold estimator.  A reuse step skips factor decompositions,
activation retention, output-gradient hooks, moment construction, and
shrinkage.  Equation~\eqref{eq:stale} guarantees strictly positive alignment
with the supplied momentum source for every cached SPD pair.  The
\(\beta_B^K\) coefficient tracks elapsed decay, while
Proposition~\ref{prop:lazy-ema} shows that sparse observation raises
stationary factor variance.  The contribution is the measured
compute--statistics tradeoff, not denoising from staleness.

\section{Experiments}

\subsection{Optimization against Muon controls}

We train three-layer causal Transformers of width 128 with four attention
heads, SwiGLU width 384, sequence length 128, and 2,048 tokens per update for
1,000 updates.  We report mean pre-update training cross-entropy in three
fixed 100-update windows and their geometric mean.  The PTB comparison uses
three hash-matched paired seeds.  Here Muon is the standard update
\(D=\polar(M)\), implemented with the same exact polar map and scalar
schedule as \method{}, at learning rate \PTBMuonLR{}.  Ranking the recorded
\method{} sweep by the
three-window geometric mean selects learning rate \PTBGoLR{}.  This
three-window ranking was introduced after the original late-window rule,
which selected \(0.16\), so we treat the hash-matched comparison as
exploratory.  The Tiny Shakespeare comparison uses three fresh paired seeds;
its frozen
campaign control is Muon with the campaign's Q/K
partner metric, headwise Q/K treatment, and registered scalar schedule.
No validation or test split is read.

\begin{table}[t]
 \caption{Training-loss trajectories against Muon controls.  Tiny Shakespeare
 uses three fresh paired seeds; PTB uses three hash-matched paired seeds.
 ``Traj. GM'' is the geometric mean of the three window means.  Lower is
 better.}
 \label{tab:muon-comparison}
 \centering
 \footnotesize
 \setlength{\tabcolsep}{4pt}
 \begin{tabular}{llccccc}
  \toprule
  Corpus & Method & Matrix LR & 201--300 & 501--600 & 901--1000 & Traj. GM \\
  \midrule
  Tiny Shakespeare
    & Muon+Q/K control & \ShakespeareMuonLR
    & \ShakespeareMuonEarly & \ShakespeareMuonMiddle
    & \ShakespeareMuonLate & \ShakespeareMuonTrajectory \\
  Tiny Shakespeare
    & \method{} (lazy \(K=4\)) & \ShakespeareGoLR
    & \textbf{\ShakespeareGoEarly} & \textbf{\ShakespeareGoMiddle}
    & \textbf{\ShakespeareGoLate} & \textbf{\ShakespeareGoTrajectory} \\
  \addlinespace
  Penn Treebank
    & Muon & \PTBMuonLR
    & \PTBMuonTunedEarly & \textbf{\PTBMuonTunedMiddle}
    & \PTBMuonTunedLate & \PTBMuonTrajectory \\
  Penn Treebank
    & \method{} (lazy \(K=4\)) & \PTBGoLR
    & \textbf{\PTBGoTunedEarly} & \PTBGoTunedMiddle
    & \textbf{\PTBGoTunedLate} & \textbf{\PTBGoTrajectory} \\
  \bottomrule
 \end{tabular}
\end{table}

On Tiny Shakespeare, \method{} is lower in all nine paired seed--window
comparisons.  Its reductions in the early, middle, and late windows are
\ShakespeareEarlyGain{}, \ShakespeareMiddleGain{}, and
\ShakespeareLateGain{}, respectively, giving a
\ShakespeareTrajectoryGain{} reduction in the trajectory geometric mean.
Peak learning rate \ShakespeareSelectedGoLR{} was selected by a broader sweep
on tuning seed 17401, where it obtains window losses
\ShakespeareSelectedGoEarly{}, \ShakespeareSelectedGoMiddle{}, and
\ShakespeareSelectedGoLate{}, with trajectory geometric mean
\ShakespeareSelectedGoTrajectory{} versus
\ShakespeareSelectedMuonTrajectory{} for its paired campaign control---a
\ShakespeareSelectedTrajectoryGain{} reduction.  Table~\ref{tab:muon-comparison}
evaluates that frozen selection on fresh seeds 17402--17404.

On PTB, the three-window-selected \method{} lowers the early and late windows by
\PTBEarlyGain{} and \PTBLateGain{}.  The middle window is effectively tied:
Muon is lower by \(0.00083\).  \method{} lowers the three-window trajectory
geometric mean by \PTBTrajectoryGain{} and wins that per-seed summary on all
three matched pairs.

\subsection{Grokking speed on corrected modular addition}

We next measure the onset of grokking rather than terminal loss.  For modulus
\(p\), the dataset contains all \(p^2\) ordered pairs and uses a deterministic
50/50 split, with the training count rounded down, into disjoint train and test
subsets.  The model is a three-layer causal
Transformer of width 64 with four attention heads, SwiGLU width 192, and RoPE.
Its input is \([a,b,\textsc{CLS}]\): the \(p\)-row numerical embedding is tied
to the \(p\)-class output head, while \textsc{CLS} is an independent learned
vector and is not an output class.  Thus the read-side control token is not
also an unused output class.

Training is full batch for 5,000 updates.  Both optimizers use momentum 0.8,
50-step linear warmup followed by a constant learning rate, and identical
AdamW updates for the tied numerical embedding, \textsc{CLS}, and norm
parameters.  \method{} uses \(K=4\) and \(\beta_B=0.97\).  A frozen screen over
learning rates \(\{0.08,0.16\}\) for \method{} and
\(\{0.16,0.32\}\) for Muon selects \ModGoLR{} and \ModMuonLR{}, respectively,
under the grokking criterion below.  The screen uses one deterministic
split and initialization seed for each modulus.

We evaluate the held-out subset every ten updates and define the grokking step
as the first of five consecutive evaluations with test accuracy at least
\(99\%\).  Table~\ref{tab:modular-grokking} also reports, as a memorization
diagnostic, the first update with \(99\%\) training accuracy and the resulting
grokking delay.  Training-set
fitting occurs at nearly the same update for the two optimizers; the separation
appears in the transition to held-out accuracy.

\begin{table}[t]
 \caption{Grokking onset on corrected independent-\textsc{CLS} modular
 addition, using one screening seed per modulus.  ``Grok'' is the first of
 five consecutive evaluations with
 held-out accuracy at least \(99\%\); evaluations occur every ten updates.
 ``Delay'' subtracts the first update with \(99\%\) training accuracy.  Lower
 is better for Train, Grok, and Delay; Speedup is the Muon/\method{} ratio of
 Grok steps.  The endpoint measures onset, while retention is separate.}
 \label{tab:modular-grokking}
 \centering
 \small
 \setlength{\tabcolsep}{4.5pt}
 \begin{tabular}{llrrrrr}
  \toprule
  Task & Method & Matrix LR & Train \(99\%\) & Grok \(99\%\) & Delay & Speedup \\
  \midrule
  mod 103 & Muon & \ModMuonLR
    & \ModOneOhThreeMuonTrain & \ModOneOhThreeMuonGrok
    & \ModOneOhThreeMuonDelay & \(1.0\times\) \\
  mod 103 & \method{} & \ModGoLR
    & \ModOneOhThreeGoTrain & \textbf{\ModOneOhThreeGoGrok}
    & \textbf{\ModOneOhThreeGoDelay}
    & \textbf{\(\ModOneOhThreeSpeedup\times\)} \\
  \addlinespace
  mod 107 & Muon & \ModMuonLR
    & \ModOneOhSevenMuonTrain & \ModOneOhSevenMuonGrok
    & \ModOneOhSevenMuonDelay & \(1.0\times\) \\
  mod 107 & \method{} & \ModGoLR
    & \ModOneOhSevenGoTrain & \textbf{\ModOneOhSevenGoGrok}
    & \textbf{\ModOneOhSevenGoDelay}
    & \textbf{\(\ModOneOhSevenSpeedup\times\)} \\
  \bottomrule
 \end{tabular}
\end{table}

On mod 103, \method{} groks at update \ModOneOhThreeGoGrok{} versus
\ModOneOhThreeMuonGrok{} for Muon, using one eighth as many updates.  On mod
107, the corresponding crossing is \ModOneOhSevenGoGrok{} versus
\ModOneOhSevenMuonGrok{}, or \(1/\ModOneOhSevenSpeedup\) as many updates.
The mod-107 \method{} trajectory later leaves the \(99\%\) region, with
below-threshold evaluations occurring as late as update
\ModOneOhSevenLastRetentionFailure{} before it returns at update
\ModOneOhSevenRetentionReturn{}.  This is a post-grok retention event and does not
redefine the time-to-grok endpoint.

\subsection{Matched map-back on Penn Treebank}

We first isolate matched map-back against the registered \emph{halfspace}
baseline.  The baseline uses an inside-only inverse-half-power score and a
spectral half-space safeguard; the matched arm uses
Equation~\eqref{eq:direction} with \(\alpha=1/4\) and the Frobenius graft.
Both use the PTB training characters only, four paired seeds, peak matrix
learning rate \(0.08\), momentum \(0.8\), and backward EMA \(0.97\).

In the preregistered updates 201--300 window, mean pre-update training loss is
\PTBHalfspaceEarly{} for halfspace and \PTBQuarterEarly{} for matched
map-back, a relative reduction of \PTBMapbackGain{}.  The matched arm is lower
on all four paired seeds.  The allocation was requeued before the final
preregistered window; in every available common 50-step terminal window,
matched map-back remains lower on all four seeds.  We do not impute the missing
901--1000 endpoint.

\subsection{Lazy refresh on Penn Treebank}

We next compare per-step quarter-power geometry with \method{} at \(K=4\)
under a fresh, complete 1,000-update paired run.  Initialization, minibatch
stream, scalar controls, and schedule are identical within each seed.
Table~\ref{tab:ptb-lazy} reports the preregistered mean pre-update training-loss
windows; no validation or test split is read by this protocol.

\begin{table}[t]
 \caption{PTB character-LM training loss over four paired seeds.  ``Wins''
 counts seeds where lazy \(K=4\) is lower than refresh 1.}
 \label{tab:ptb-lazy}
 \centering
 \small
 \begin{tabular}{lrrr}
  \toprule
  Update window & Refresh 1 & Lazy \(K=4\) & Lazy wins \\
  \midrule
  201--300  & \PTBRefreshEarly & \PTBLazyEarly & 2/4 \\
  501--600  & \PTBRefreshMiddle & \textbf{\PTBLazyMiddle} & 2/4 \\
  901--1000 & \PTBRefreshLate & \textbf{\PTBLazyLate} & 4/4 \\
  \bottomrule
 \end{tabular}
\end{table}

The early window is operationally tied.  Lazy refresh is lower by
\PTBLazyMiddleGain{} in the middle window and \PTBLazyLateGain{} in the late
window, winning all four late pairs.  Its geometric-mean step-time ratio is
\PTBLazyStepRatio{}, a reduction of \PTBLazyStepSaving{}.  Both arms record
zero negative momentum-source and current-gradient alignments.  These results
show that the noisier sparse factor estimator can coexist with a favorable
training trajectory; they do not identify blockwise coherence as the cause.

\section{Conclusion}

The first-principles object in \method{} is an SPD-weighted spectral oracle:
\[
 \boxed{
 C=P_BMP_A,\qquad
 D_0=P_BQ(C)P_A,\qquad
 \ip{M}{D_0}=\|C\|_*.
 }
\]
This identity holds for every nonzero source and every supplied SPD pair,
without Fisher or freshness assumptions.  \method{} chooses
\(P_A=A^{-1/4}\) and \(P_B=\bar B^{-1/4}\) from OAS observed-label
activation/backward factors; under predictive consistency, the backward
factor approaches the softmax Fisher/GGN factor.  The Frobenius graft retains
the oracle ray while quarter powers limit worst-case sensitivity to factor
conditioning.

Four-step refresh is a separate systems and statistical choice.  It preserves
elapsed decay and slow-signal lag, increases stationary factor variance, and
removes factor capture and matrix functions from three of every four updates.
On the frozen PTB training protocol, the three-window-selected \method{}
improves the hash-matched trajectory summary over Muon, while the refresh
trade reduces mean step time and improves the late training window.  The
corrected modular-addition screen shows a distinct empirical pattern on its
screening seeds: training-set fitting time is nearly unchanged, while
sustained-\(99\%\) grokking onset is reached at updates
\ModOneOhThreeGoGrok{} and \ModOneOhSevenGoGrok{}, versus
\ModOneOhThreeMuonGrok{} and \ModOneOhSevenMuonGrok{} for Muon.  The resulting
picture is explicit:
matched map-back provides coordinate-consistent spectral optimization,
factor estimation chooses the geometry, and lazy refresh amortizes its cost.

\bibliographystyle{iclr2026_conference}
\bibliography{oas_references}

\appendix

\section{Proofs and additional guarantees}
\label{app:proofs}

\subsection{Proof of Proposition~\ref{prop:conditioning}}

Write \(P_B=B^{-\alpha}\), \(P_A=A^{-\alpha}\), and let \(Q\) be the compact
full-SVD support polar, so \(\|Q\|_F=\sqrt{k}\).  Nuclear-norm multiplication
inequalities give
\[
 \|P_BMP_A\|_*
 \geq\sigma_{\min}(P_B)\sigma_{\min}(P_A)\|M\|_*,
\]
while
\[
 \|P_BQP_A\|_F
 \leq\sqrt{k}\|P_B\|_{\op}\|P_A\|_{\op}.
\]
Substitution into
\[
 \ip{M}{\widehat D_\alpha}
 =\frac{\sqrt{k}\|P_BMP_A\|_*}{\|P_BQP_A\|_F}
\]
gives
\[
 \ip{M}{\widehat D_\alpha}
 \geq
 \frac{\sigma_{\min}(P_B)\sigma_{\min}(P_A)}
 {\sigma_{\max}(P_B)\sigma_{\max}(P_A)}
 \|M\|_*
 =K_{AB}^{-\alpha}\|M\|_*.
\]
The nonzero singular values of a semi-orthogonal \(Q\) are all one.
Singular-value product inequalities therefore give
\[
 \sigma_{\max}(P_BQP_A)
 \leq\sigma_{\max}(P_B)\sigma_{\max}(P_A),
\]
\[
 \sigma_k(P_BQP_A)
 \geq\sigma_{\min}(P_B)\sigma_{\min}(P_A).
\]
Their ratio is at most \(K_{AB}^{\alpha}\).  The positive graft does not
change this ratio.

For completeness, Equation~\eqref{eq:oas-condition} follows by writing
\(\mu=\tr(S)/d\).  Since \(S\succeq0\),
\[
 \lambda_{\min}(\operatorname{OAS}(S))\geq\lambda\mu,\qquad
 \lambda_{\max}(\operatorname{OAS}(S))
 \leq[d-(d-1)\lambda]\mu.
\]

\subsection{A scale-free cached-oracle bound}

\begin{proposition}[Cached oracle under coordinate drift]
\label{prop:stale-approx}
Let \(q_P\) be the current weighted norm from
Equation~\eqref{eq:weighted-norm}, and let
\(\widetilde D\) be any positive rescaling of an exact unit-radius oracle for
cached maps \(\widetilde P_B,\widetilde P_A\).  Define \(\Theta\) by
Equation~\eqref{eq:stale-distortion}.  Then
\begin{equation}
 \boxed{
 \frac{\ip{M}{\widetilde D}}
 {q_P(\widetilde D)\|P_BMP_A\|_*}
 \geq\Theta^{-1}.}
 \label{eq:stale-approx}
\end{equation}
Thus, after rescaling to its actual radius in the current ball, the cached
direction achieves at least a \(\Theta^{-1}\) fraction of the current optimum.
The ratio is unchanged by the Frobenius graft.
\end{proposition}

\begin{proof}
Let
\[
 \widetilde q(D)
 =\|\widetilde P_B^{-1}D\widetilde P_A^{-1}\|_{\op},
\quad
 L=P_B^{-1}\widetilde P_B,
\quad
 R=\widetilde P_AP_A^{-1}.
\]
For \(U=\widetilde P_B^{-1}D\widetilde P_A^{-1}\),
\(q_P(D)=\|LUR\|_{\op}\), and hence
\[
 a\,\widetilde q(D)\leq q_P(D)\leq b\,\widetilde q(D),
\]
where
\[
 a=\sigma_{\min}(L)\sigma_{\min}(R),\qquad
 b=\sigma_{\max}(L)\sigma_{\max}(R).
\]
The corresponding unit-ball inclusions imply that the cached support value
\(\widetilde h(M)=\|\widetilde P_BM\widetilde P_A\|_*\) satisfies
\(\widetilde h(M)\geq a\|P_BMP_A\|_*\).  An exact cached unit oracle has
pairing \(\widetilde h(M)\) and current radius at most \(b\).  Its normalized
efficiency is therefore at least \(a/b\).  Since
\(\kappa_2(R)=\kappa_2(P_A^{-1}\widetilde P_A)\),
\(a/b=\Theta^{-1}\).  Positive rescaling cancels from the ratio.
\end{proof}

The bound uses the singular condition numbers of the actual coordinate
changes.  A simpler factor-level power law requires additional commutativity
or functional-calculus assumptions and is not used here.

\subsection{Proof of Proposition~\ref{prop:lazy-ema}}

Let \(F_X(d\omega)\) be the matrix-valued spectral measure of the centered
vector process \(\operatorname{vec}(X_t)\).  At refresh boundaries, the sparse
and per-step transfer functions with respect to the original time index are
\[
 G_K(\omega)=
 \frac{1-\beta_B^K}{1-\beta_B^Ke^{-iK\omega}},
 \qquad
 H_1(\omega)=
 \frac{1-\beta_B}{1-\beta_Be^{-i\omega}}.
\]
Factorizing \(1-(\beta_Be^{-i\omega})^K\) gives
\[
 \frac{|G_K(\omega)|}{|H_1(\omega)|}
 =
 \frac{\sum_{r=0}^{K-1}\beta_B^r}
 {\left|\sum_{r=0}^{K-1}\beta_B^re^{-ir\omega}\right|}
 \geq1.
\]
Therefore
\[
 \operatorname{Cov}(\operatorname{vec}Y_j)
 -\operatorname{Cov}(\operatorname{vec}Z_{jK})
 =
 \int_{-\pi}^{\pi}
 (|G_K|^2-|H_1|^2)F_X(d\omega)
 \succeq0.
\]

For IID observations, the standard AR(1) variance formula immediately yields
Equation~\eqref{eq:iid-variance}.  For an affine signal
\(x_t=a+vt\), the steady refresh-boundary estimate is
\[
 Y_j=x_{jK}-v\frac{K\beta_B^K}{1-\beta_B^K}.
\]
Holding this value through phase \(s\in\{0,\ldots,K-1\}\) adds \(s\) steps of
delay.  Averaging over phases gives Equation~\eqref{eq:lazy-lag}.

\end{document}